\documentclass[letterpaper, 10 pt, conference]{ieeeconf}  

\IEEEoverridecommandlockouts                              

\usepackage{times} 
\usepackage{amsmath} 
\usepackage{amssymb,amsfonts}  
\usepackage[sort,compress]{cite}
\usepackage{graphicx}
\usepackage{booktabs}
\usepackage{mathrsfs}
\usepackage{subcaption}
\usepackage{tikz}
\usepackage{algorithm,balance}
\usepackage[noend]{algpseudocode}

\newcommand*\mc[1]{\ensuremath{\mathcal{#1}}}  
\newcommand*\ttt[1]{\texttt{#1}} 

\newcommand*\bigO{\ensuremath{\mathcal{O}}}  
\newcommand*\R{\ensuremath{\mathbb{R}}}  
\newcommand*\SE{\ensuremath{\text{SE}}}  
\newcommand*\argmin[1]{\underset{#1}{\mathrm{argmin}} \mathop{}}  
\newcommand*\BT{\ensuremath{\mc{T}_{\tt{BT}}}} 

\newcommand*\mx[1]{\mathbf{#1}} 
\newcommand*\vc[1]{\mathbf{#1}} 

\newtheorem{lemma}{Lemma}

\usetikzlibrary{backgrounds}
\tikzstyle{pnode} = [draw,circle,minimum width=0.25cm,minimum height=0.25cm,fill=cyan]
\tikzstyle{lnode} = [draw,circle,minimum width=0.25cm,minimum height=0.25cm,fill=green]
\tikzstyle{wlabel} = [shift={(180:0.4)}]
\tikzstyle{slabel} = [shift={(270:0.4)}]
\tikzstyle{nlabel} = [shift={(90:0.4)}]

\renewcommand{\baselinestretch}{.98}
\title{\LARGE \bf
Complexity Analysis and Efficient Measurement Selection Primitives for High-Rate Graph SLAM
}

\author{Kristoffer M. Frey$^{1}$, Ted J. Steiner$^{2}$, and Jonathan P. How$^{1}$
\thanks{Work was supported by the Defense Advanced Research Projects Agency (DARPA) as part of the Fast Lightweight Autonomy (FLA) program, HR0011-15-C-0110.
        The views expressed here are those of the authors, and do not reflect the official views or policies of the Department of Defense or the U.S. Government.}
\thanks{$^{1}$Kristoffer M. Frey and Jonathan P. How are with the Department of Aeronautics and Astronautics, MIT,
        Cambridge, MA 02139, USA
        {\tt\small kfrey@mit.edu}
        {\tt\small jhow@mit.edu}}%
\thanks{$^{2}$Ted J. Steiner is a Senior Member of the Technical Staff at Draper,
        Cambridge, MA 02139, USA
        {\tt\small tsteiner@draper.com}}%
}

\begin{document}

\vspace{1pt}
\maketitle
\thispagestyle{empty}
\pagestyle{empty}

\begin{abstract}

Sparsity has been widely recognized as crucial for efficient optimization in graph-based SLAM.
Because the sparsity and structure of the SLAM graph reflect the set of incorporated measurements, many methods for sparsification have been proposed in hopes of reducing computation.
These methods often focus narrowly on reducing edge count without regard for structure at a global level.
Such structurally-na\"{\i}ve techniques can fail to produce significant computational savings, even after aggressive pruning.
In contrast, simple heuristics such as measurement decimation and keyframing are known empirically to produce significant computation reductions.
To demonstrate why, we propose a quantitative metric called \emph{elimination complexity} (EC) that bridges the existing analytic gap between graph structure and computation.
EC quantifies the complexity of the primary computational bottleneck: the factorization step of a Gauss-Newton iteration.
Using this metric, we show rigorously that decimation and keyframing impose favorable global structures and therefore achieve computation reductions on the order of $r^2 / 9$ and $r^3$, respectively, where $r$ is the pruning rate.
We additionally present numerical results showing EC provides a good approximation of computation in both batch and incremental (iSAM2) optimization and demonstrate that pruning methods promoting globally-efficient structure outperform those that do not.
\end{abstract}

\section{Introduction}

Graph-based approaches to the Simultaneous Localization and Mapping (SLAM) problem have gained popularity in recent years \cite{dellaert2006square,kaess2011isam2,sibley2010sliding,huang2013consistent,steiner2017vision,khosoussi2016maximizing}.
These problems are generally formulated as nonlinear least-squares (NLLS) optimizations over the set of robot poses and landmark positions.
Though in certain cases non-iterative solutions exist \cite{rosen2016se}, SLAM problems are usually solved iteratively using a form of Gauss-Newton (GN).
As recognized in \cite{dellaert2006square}, SLAM problems demonstrate a naturally sparsity that can be leveraged to significantly reduce computation with each GN iteration.
This sparsity manifests itself as a large number of zero entries in the graph adjacency matrix, or, equivalently, as a large number of ``missing'' edges relative to a complete graph.
Nevertheless, this paper will emphasize that computation is a function of the graph structure and not simply just edge count.

Because measurements correspond to edges in the SLAM graph, the choice of measurements included in the optimization thus directly influences sparsity and graph structure.
This has recently motivated numerous sophisticated measurement-selection and sparsification strategies for computation reduction~\cite{carlevaris2013long,huang2013consistent,mazuran2016nonlinear,khosoussi2016maximizing}.
However, these methods narrowly focus on reducing edge count (i.e., treat all edges as equivalent from a computation perspective) \cite{dissanayake2000computationally,huang2013consistent,khosoussi2016maximizing} or enforcing locally-sparse structure \cite{carlevaris2013long,mazuran2016nonlinear}.
Because they do not curate global structure, these methods may only achieve limited computational savings, even after aggressive pruning.

To better connect graph structure to computation, this paper proposes \emph{elimination complexity} (EC) as an approximation of the FLOP count associated with each GN iteration.
As a function purely of graph structure, EC is independent of the particular numeric values of the optimization at any time.
Experimental evidence confirms that EC trends linearly with computation time in both batch and incremental modes of operation.

We apply EC-based analysis to keyframing and decimation, two common measurement selection primitives in high-rate vision-based SLAM \cite{ila2010information,stalbaum2013keyframe,carlone2017attention,mur2017orb,steiner2017vision}.
These two heuristics are shown to produce \emph{globally} efficient graph structures that achieve significant computational savings.
In particular, it is demonstrated both analytically and numerically that decimation reduces EC at a rate of $\sim r^2 / 9$, while keyframing reduces EC at an increased rate of $r^3$.
Here, $r$ refers to ``rate'' of pruning in either case and is defined in more detail in the body of this paper.
These insights can guide application of these techniques in systems with significant computation constraints and high required optimization rates.

As an analytic tool, EC provides a key link between graph structure and computation that has thus far been lacking from the discussion of measurement selection in SLAM.
Furthermore, because it can be evaluated directly from the iSAM2 Bayes Tree \cite{kaess2011isam2}, EC can potentially be used in an online sense to adapt computation management policies as the SLAM graph evolves.

In practice, measurement selection ultimately aims to tradeoff accuracy for computational savings.
While quantitative connections between graph structure and various accuracy metrics have been made for linear SLAM graphs \cite{olson2009evaluating,khosoussi2014novel}, their extension to general SLAM measurement models is unclear.
Rather than attempting to address the accuracy side of this tradeoff, this paper aims to establish that achieving good computation reduction is not simply a matter of maximizing the number of measurements removed, but rather about imposing favorable global graph structure.
Thus, techniques which select measurements for removal based on their estimation ``value'' but without respect for global structure may fail to produce the desired level of computational savings.

\subsection{Related Work}

The most relevant discussion of computation in graph SLAM comes from the seminal paper by Dellaert et al.~\cite{dellaert2006square}, and the iSAM2 paper by Kaess et al.~\cite{kaess2011isam2}.
There, the authors connect GN optimization to sparse factorization of the corresponding system matrix.
They also make the important observation that the choice of elimination ordering is crucial for efficient optimization.
However, they do not explicitly connect the underlying Bayes Tree structure \cite{kaess2011isam2} to a metric of computational complexity.
While the complexity of sparse matrix factorization is well-understood in the linear algebra literature \cite{rose1972graph,heggernes1996finding,luce2014minimum}, such results have not been tailored for SLAM-specific solvers nor applied to measurement selection in SLAM.

Computation is often managed in SLAM by bounding the number of variables being estimated, usually via sliding-window approaches \cite{steiner2017vision,chiu2013robust} or by merging nearby pose variables \cite{ila2010information,johannsson2013temporally,carlevaris2013long}.
Nonetheless, as measurements (and corresponding edges) are added the graph grows denser, and further steps must often be taken to promote sparsity.

Many methods of measurement selection exist in the literature with the specific goal of promoting sparsity, though none are guaranteed to produce computationally-efficient global graph structures.
The Sparse Extended Information Filter \cite{dissanayake2000computationally} or the $\mc{L}_1$ regularization of \cite{huang2013consistent} directly sparsify the information matrix, although without regard for resulting structure.
Alternatively, \cite{carlevaris2013long,mazuran2016nonlinear} make use of pre-selected sparse topologies to replace dense regions of the graph, and structural considerations are limited to local regions of the graph.
Heuristics for applying these sparsification optimizations in online pose-graph SLAM were suggested in \cite{carlevaris2013long}, but applying such techniques to the more general graphs common in landmark-SLAM is less straightforward.
Following a general graph-theoretic approach, \cite{khosoussi2016maximizing} provides max-spanning-tree-algorithms which select a fixed number of edges to retain or prune.
Besides the fact that none of these approaches explicitly consider global structure from a computational standard, none achieve both practicality for computationally-constrained systems and direct applicability to general landmark-SLAM graphs.
In contrast, keyframing and decimation policies are generally inexpensive and have been used extensively in real-time SLAM \cite{ila2010information,stalbaum2013keyframe,wang2013kullback,carlone2017attention}.

\section{Estimating Computation}

The graph SLAM optimization problem with Gaussian factors can be expressed as a NLLS minimization \cite{dellaert2006square}
\begin{equation}\label{eq:nlls}
  \argmin{X} \frac{1}{2}\sum_{i = 0}^m (\vc{h}_i(X_i) - \vc{z}_i)^T \mx{\Sigma}_i^{-1} (\vc{h}_i(X_i) - \vc{z}_i).
\end{equation}
Using iterative GN \cite{bertsekas1999nonlinear}, optimization of \eqref{eq:nlls} requires multiple solves of the (assumed positive-definite) \emph{linearized} system
\begin{equation}\label{eq:normal_eqns}
  \mx{A}^T \mx{A} \vc{x} = \mx{A}^T \vc{b},
\end{equation}
The graph representations of the original \eqref{eq:nlls} and linearized systems \eqref{eq:normal_eqns} are identical, meaning sparsity is preserved through linearization.

The number of GN iterations required for convergence depends fundamentally on the initial values, measurement values (which are assumed random), and environmental factors such as availability of opportunistic landmarks \cite{bertsekas1999nonlinear}.
Thus it is difficult in general to predict how many iterations will be needed, or even if convergence will occur.

In contrast, linear systems of the form \eqref{eq:normal_eqns} are well-characterized, and computation is a function of the corresponding graph structure rather than numeric values.
As the total GN computation is essentially the sum of these linear solves, computational savings at the linear level corresponds to multi-fold savings in total.
Thus, quantifying the computation involved in \eqref{eq:normal_eqns} provides a link from structure to total computation.

The two fundamental steps involved in solving the linear system \eqref{eq:normal_eqns} are \emph{elimination} and \emph{back-substitution}.
Elimination is equivalent to factorization of the system $\mx{A}^T \mx{A} = \mx{R}^T \mx{R}$ into the upper-triangular square matrix $\mx{R}$ \cite{dellaert2006square}.
Back-substitution refers to solving the remaining triangular system, and has complexity linear in the number of non-zero elements in $\mx{R}$.
Because elimination of the $n \times n$ system $\mx{A}^T \mx{A}$ carries a worst-case $\bigO(n^3)$ complexity, it often represents the majority of computation in practice, even for incremental algorithms \cite{kaess2011isam2}.
This motivates the use of \emph{elimination complexity} as a representative measure of the inherent complexity represented by a graph.

\subsection{Elimination Complexity}\label{s:elimination}

\begin{figure*}[ht]
  \centering
  \includegraphics[width=0.8\textwidth]{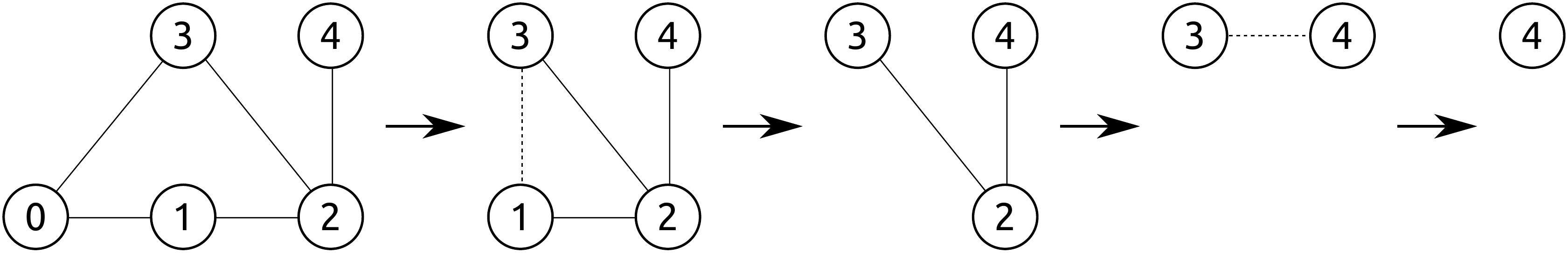}
  \caption[Node elimination on a simple graph]{
    The node elimination algorithm executed on a simple graph $G^{(0)}$.
    Nodes are eliminated in the order (0, 1, 2, 3, 4), producing a series of elimination graphs $G^{(i)}$.
    Induced edges are shown with dotted lines.
  }
  \label{fig:node_elim}
  \vspace*{-.25cm}
\end{figure*}

As noted by \cite{dellaert2006square,kaess2010bayes}, \eqref{eq:normal_eqns} can be solved efficiently by sparse factorization of $\mx{A}^T \mx{A}$.
Sparse factorization follows the pattern of \emph{node elimination} on the graph, illustrated in Figure~\ref{fig:node_elim}.
In node elimination, variable nodes are eliminated one-by-one, corresponding to marginalization of the variables from the joint distribution over all remaining variables.
When the $i$-th node is eliminated, it is removed from the elimination graph and edges are induced such that all its remaining neighbors form a fully-connected clique in elimination graph $G^{(i+1)}$.
These new edges did not exist in the original graph and thus constitute \emph{fill}, representing intermediate dependencies between variables induced by a particular elimination ordering.
In the final upper-triangular $\mx{R}$ factor, these fill edges correspond to nonzero ``filled-in'' entries that were zero in the original system matrix $\mx{A}^T \mx{A}$.

As is well-known from the sparse linear algebra literature \cite{dellaert2006square,davis2004algorithm,yannakakis1981computing}, fill depends on the chosen variable ordering $\mc{P}$.
Though the solution itself is unaffected by ordering, different orderings can result in widely differing fill at each step of the optimization.
However, determining the optimal (min-fill or min-FLOP) ordering is NP-complete \cite{yannakakis1981computing,luce2014minimum}.
In practice, efficient heuristics such as Column-Approximate Min-Degree (\texttt{COLAMD}) \cite{davis2004algorithm} are widely used \cite{dellaert2006square,kaess2011isam2}.

For certain graph structures, na\"{\i}ve elimination ordering can produce catastrophic fill-in, even if the original graph was quite sparse.
For example, Duff \cite{duff1974number} showed that in random matrices with initially very few non-zero entries, as elimination proceeds the probability of the remaining elements becoming non-zero (due to fill-in) rapidly approaches one.

Each step of node elimination corresponds to computing one step of the corresponding sparse QR or Cholesky factorization \cite{rose1972graph,heggernes1996finding,dellaert2006square}, and involves computation scaling with the dimension of the neighbors of the eliminated node. From this perspective, solving the full system \eqref{eq:normal_eqns} is decomposed into solving a series of smaller, dense sub-problems. The complexity of factorizing the full sparse system is then simply the sum of the complexities of the individual dense sub-problems.

Rose \cite{rose1972graph} showed that computing the $\mx{R}^T \mx{R}$ decomposition of a sparse $n \times n$ matrix can be performed in
\begin{equation}\label{eq:rose_cholesky_mults}
  \frac{1}{2} \sum_{i=1}^{n-1} d(i, \mc{P}) (d(i, \mc{P}) + 3) \sim \sum_{i=1}^{n-1} d(i, \mc{P})^2
\end{equation}
multiplications, where $d(i, \mc{P})$ refers to the degree of the $i$-th eliminated node in the elimination graph $G^{(i)}$ produced by ordering $\mc{P}$.
The asymptotic form of \eqref{eq:rose_cholesky_mults} is equivalent to the approximate Cholesky FLOP count  \cite{luce2014minimum}.
Note that, for a fully dense matrix (corresponding to a fully-connected graph), $d(i) \sim n$ (full system dimension) and sparse factorization approaches the full $n^3$ complexity for dense matrices.

In a conventional linear algebra approach, factorization of the system \eqref{eq:normal_eqns} occurs one row (or column) at a time.
The corresponding graph $G^{(0)}$ includes $n$ nodes, matching the scalar dimension of the system.
However, \eqref{eq:normal_eqns} has additional block structure for typical SLAM systems~\cite{dellaert2006square}.
In SLAM, the variables of interest are often multi-dimensional quantities such as positions and rotations, and measurements generally are defined on the level of these ``macro-variables.''
In this case, it is the \emph{block} sparsity pattern of \eqref{eq:normal_eqns} that is represented in the factor graph.

By applying ordering heuristics such as \texttt{COLAMD} \cite{davis2004algorithm} on the block structure directly, \cite{dellaert2006square} showed improved performance and less fill.
Following this observation, modern SLAM solvers such iSAM2 \cite{kaess2011isam2} apply elimination directly on the ``macro-variables'' of the factor graph.
This motivates the definition of a version of \eqref{eq:rose_cholesky_mults} that accounts for the block structure of SLAM.

We define the \emph{elimination complexity} (EC) $\mc{C}(G, \mc{P})$ of a factor graph $G$ with variables $\mc{X}$ and ordering $\mc{P}$ as
\begin{equation}\label{eq:elim_complexity}
    \mc{C}(G, \mc{P}) \triangleq \sum_{i = 1}^{\lvert \mc{X} \rvert} d_f(i) \Big( d_f(i) + d_s(i, G, \mc{P}) \Big)^2
\end{equation}
where $d_f(i, \mc{P})$ and $d_s(i, G, \mc{P})$ are the total scalar dimension of the $i$-th frontal variable $\vc{x}_f$ and its corresponding separator set $\vc{x}_s$ according to $G$ and $\mc{P}$, respectively.
Note that due to induced fill and the removal of previously-eliminated nodes, the dimension $d_s$ of the $i$-th-eliminated variable will generally not match the neighborhood of that variable in the original graph $G$.
Thus, computation of $\mc{C}$ in general requires simulation of the elimination process using ordering $\mc{P}$.
Under scalar elimination, which corresponds to frontal variables of singular dimension $d_f(i) = 1$, the EC \eqref{eq:elim_complexity} reduces to the asymptotic form of \eqref{eq:rose_cholesky_mults}.

\begin{lemma}\label{lemma:C_non_decreasing}
  For a fixed elimination ordering $\mc{P}$ and graph $G$, let $G^+$ be constructed by adding an edge to $G$.
  Then, $\mc{C}(G, \mc{P}) \leq \mc{C}(G^+, \mc{P})$.
\end{lemma}

\begin{proof}
	Let $G^+$ refer to the graph constructed by adding an edge to $G$.
	Following the elimination process described in Section~\ref{s:elimination}, the elimination neighborhood at each step $i$ cannot be smaller for $G^+$ than for $G$
	\begin{equation}
		d_s(i, G, \mc{P}) \leq d_s(i, G^+, \mc{P}) \qquad \forall i
	\end{equation}
	Substituting this into the definitions of $\mc{C}(G, \mc{P})$ and $\mc{C}(G^+, \mc{P})$ yields
$		\mc{C}(G, \mc{P}) \leq \mc{C}(G^+, \mc{P})$.
\end{proof}
Lemma~\ref{lemma:C_non_decreasing} confirms the intuition that, for a fixed ordering, adding an edge to the graph cannot decrease elimination complexity.
Equivalently, removing an edge cannot \emph{increase} complexity, which confirms the conventional intuition applied by many existing measurement pruning techniques.
Importantly, however, it will be demonstrated experimentally that arbitrary edge pruning (with no regard for global structure) can be surprisingly \emph{ineffective} at reducing complexity, and thus EC-naive pruning techniques may only achieve limited computation reduction even after aggressive pruning.

\subsection{Relationship to the iSAM2 Bayes Tree}\label{s:bayes_tree}

Characterizing the update-time computation of incremental solvers such as iSAM2 \cite{kaess2011isam2} is generally difficult.
An iSAM2 incremental update attempts to avoid re-elimination of the entire graph by representing the current solution and elimination process in a Bayes Tree structure.
Because it performs re-linearization and re-ordering only as needed, the computation required for each update depends on the the full update history and numerics of the problem.
Nonetheless, because the Bayes Tree fundamentally represents the elimination process of the system, elimination complexity still serves as a useful predictor of computation.

EC can be computed directly from the Bayes Tree, although some subtleties must be acknowledged.
Though generally guided by a \ttt{COLAMD} ordering, the implicit variable ordering represented in the Bayes Tree is semi-static and may not necessarily match $\mc{P}_{\ttt{COLAMD}}(G)$ at any time.
Furthermore, cliques in the Bayes Tree represent a \emph{multifrontal} factorization~\cite{dellaert2005multifrontal}, in which multiple variables may in certain cases be eliminated together rather than sequentially, taking advantage of optimized dense matrix operations.

For a given iSAM2 instance represented by the Bayes Tree $\BT$, elimination complexity can be computed
\begin{equation}\label{eq:elim_complexity_bayes_tree}
    \mc{C}_{\ttt{BT}}(\BT) \triangleq \sum_{C \in \BT} d_f(C) \Big( d_f(C) + d_s(C) \Big)^2
\end{equation}
where $d_f(C)$ and $d_s(C)$ are the total scalar dimension of the frontal and separator variables in clique $C$, respectively.

\begin{figure}[t]
	\centering
	\includegraphics[width=\columnwidth,trim={4cm, 8.5cm, 4.5cm, 8cm},clip]{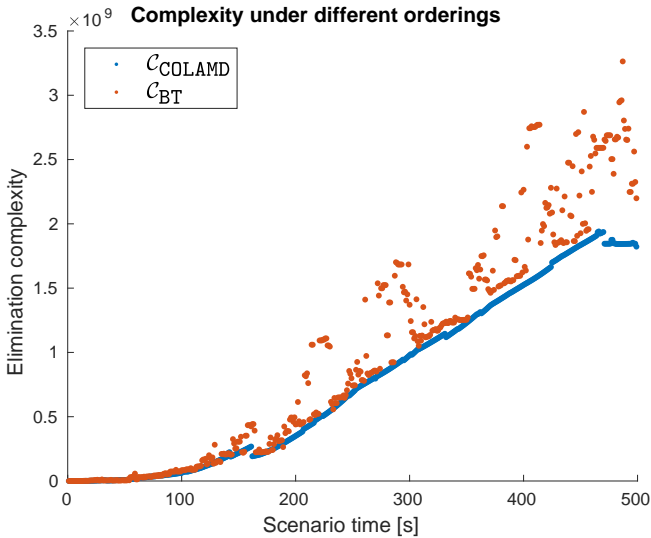}
	\caption{
		Elimination complexity (EC) grows over time during a simulated SLAM experiment as more variables and measurements are added.
		$\mc{C}_{\ttt{COLAMD}}$ is computed from the graph using a $\ttt{COLAMD}$ variable ordering, and $\mc{C}_{\ttt{BT}}$ is computed directly from the Bayes Tree.
		iSAM2 uses a constrained form of \ttt{COLAMD} and only updates the ordering when nodes are re-eliminated.
		For this reason $\mc{C}_{\ttt{BT}}$ is often greater than $\mc{C}_{\ttt{COLAMD}}$.
	}
	\label{fig:elim_complexity_results_1}
	\vspace{-0.25cm}
\end{figure}

\begin{figure}[t]
	\centering
	\includegraphics[width=\columnwidth,trim={4cm, 8.5cm, 4.5cm, 8cm},clip]{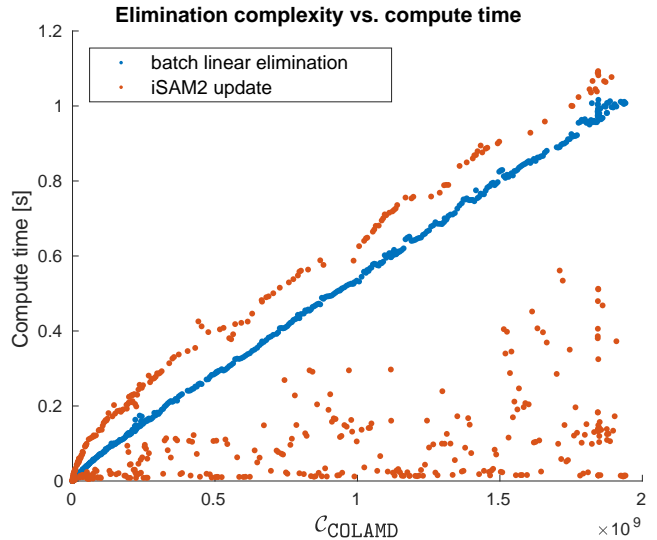}
	\caption{
		Our elimination complexity (EC) metric, $\mc{C}_{\ttt{COLAMD}}$, is directly proportional to the batch elimination time of the linearized system (blue).
		Here iSAM2 (red) often performed the full nonlinear update without fully re-eliminating the graph, resulting in a greatly reduced compute time.
		However, the in-practice worst-case computation involves re-eliminating the entire graph and tracks nearly linearly with EC.
	}
	\label{fig:elim_complexity_results_2}
	\vspace{-0.25cm}
\end{figure}

\subsection{Numerical Results}

As a measure of graph computational complexity, EC $\mc{C}(G, \mc{P})$ should correlate linearly with the actual computation time.
This was verified in an incremental SLAM simulation using iSAM2 and with the elimination complexity evaluated using the $\ttt{COLAMD}$ heuristic.
Figure \ref{fig:elim_complexity_results_1} shows that both $\mc{C}_{\ttt{COLAMD}}$ and $\mc{C}_{\ttt{BT}}$ grow over time as more poses, landmarks, and measurements are added to the SLAM system.
iSAM2 employs a lazy re-ordering scheme that attempts to maintain a near-$\ttt{COLAMD}$ ordering, which often makes $\mc{C}_{\ttt{BT}}$ larger than $\mc{C}_{\ttt{COLAMD}}$.

Figure~\ref{fig:elim_complexity_results_2} shows that the batch elimination time of the linearized system at each step is proportional to $\mc{C}_{\ttt{COLAMD}}$.
Furthermore, the incremental update time of iSAM2 displays an ``in-practice'' worst-case computation which also follows $\mc{C}_{\ttt{COLAMD}}$ linearly.
Though iSAM2 often avoids re-eliminating much of the graph in order to produce relatively low-cost updates, in this experiment it still re-eliminated much or all of the graph the majority of the time.
These results demonstrate that elimination complexity provides an approximation of computation even for incremental solvers.
This motivates the use of EC as an analytic link between graph structure and computation, allowing for quantitative evaluation of measurement selection strategies.

\section{Measurement Selection Primitives}\label{s:primitives}

Armed with a measure of complexity relating graph structure to computation, we can now assess measurement selection strategies analytically for their affects on computation.
We focus on the two most ubiquitous selection policies in use in high-rate SLAM: decimation and keyframing.
First, we will define a general case of the landmark-SLAM graph and approximate the worst-case elimination complexity.
Then we will examine both keyframing and decimation, providing quantitative estimates of computation reduction.

\subsection{The Landmark-SLAM Graph}

A typical landmark-SLAM problem is shown in Figure \ref{fig:lm_SLAM}.
Landmarks are represented as green nodes $l_j$, and robot poses (sampled along the trajectory) as blue nodes $x_i$.
As is typical of mobile robots, odometry measurements relate sequential poses, and could arise from inertial sensors or wheel encoders.
Landmark observations connect the robot's pose at a particular moment in time to a particular landmark, and could represent measurements taken from LIDAR, ranging data, or a camera image.

Due to real-world sensor limitations, most landmarks will only be observable from a small subset of poses.
However, the following analysis will assume the worst-case (from a sparsity perspective) -- that every landmark is observed from every pose, as illustrated in Figure \ref{fig:lm_SLAM}.
As is standard in SLAM approaches \cite{mourikis2007multi,steiner2017vision,martinelli2012vision}, it is assumed that landmark positions are uncorrelated \emph{a priori}, and thus no landmark-landmark edges exist in the graph.
Furthermore, to simplify the following discussion, direct pose-pose loop closures are not explicitly taken into account here.
Nonetheless, generalizing these results to allow for this form of loop closure is straightforward.

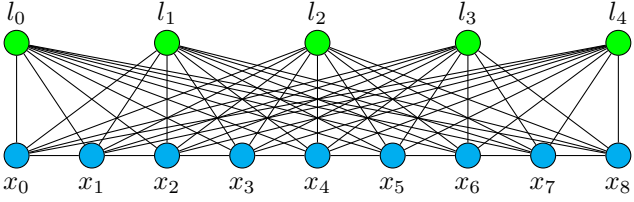
\begin{figure}[t]
  \begin{tikzpicture}
    \foreach \i in {0,1,...,8} {
      \node [pnode, anchor=east] (x-\i) at (\i,0) {};
      \node [slabel] at (x-\i) {$x_{\i}$};
    }
    \foreach \j in {0,1,...,4} {
      \node [lnode, anchor=east] (l-\j) at (\j*2,1.5) {};
      \node [nlabel] at (l-\j) {$l_{\j}$};
    }
    \foreach \i[count=\inext from 1] in {0,1,...,7} {
      \draw (x-\i) -- (x-\inext) {};
    }
    \foreach \i in {0,1,...,8} {
      \foreach \j in {0,1,...,4} {
        \draw (x-\i) -- (l-\j);
      }
    }
  \end{tikzpicture}
  \caption{
    A small landmark-SLAM example, in which every landmark $l_j$ is observed from every pose $x_i$.
    Odometry constraints (e.g., from inertial sensors or wheel encoders) connect consecutive poses.
    Assuming no direct pose-pose loop closures or direct landmark-landmark correlations are allowed (as is typical of visual-inertial odometry systems), this is the \emph{worst-case} graph from a sparsity perspective.
  }
  \label{fig:lm_SLAM}
  \vspace{-0.25cm}
\end{figure}

Let there be $n_l$ landmarks and $n_x$ poses in landmark-SLAM graph $G$.
The elimination complexity of this form of landmark-SLAM can be estimated by simulating the elimination process using a heuristic landmark-then-pose ordering $\bar{\mc{P}}$.
In practice, the \emph{de facto} \ttt{COLAMD} ordering is based on a min-degree heuristic, and will depend on the realized graph structure, which in turn depends on sensor limitations, the distribution of landmarks in the environment, and the robot's trajectory.
However, by analyzing the worst-case graph in which all landmarks are observed by all poses, and assuming that $n_l > n_x$, then $\bar{\mc{P}} = \mc{P}_{\ttt{min\ degree}} \approx \mc{P}_{\ttt{COLAMD}}$.
Thus, it follows that
\begin{equation}\label{eq:complexity_G}
  \mc{C}(G, \bar{\mc{P}}) \sim (d_l n_l + d_x n_x) (d_x n_x)^2
\end{equation}
where $d_x$ and $d_l$ are the scalar dimensions of the pose and landmark variables, respectively.

To derive \eqref{eq:complexity_G}, note that eliminating each of the $n_l$ landmarks incurs a complexity of $d_l (d_l + n_x d_x)^2$, as each landmark is adjacent to all $n_x$ poses.
Because the landmark nodes are non-adjacent in $G$, fill is only induced between pose nodes, leaving a fully-connected clique of the $n_x$ pose nodes.
This fully-connected clique then can be eliminated in $\bigO( d_x^3 n_x^3 )$ operations, producing the total asymptotic operation count shown in \eqref{eq:complexity_G}.

Also note that the inherent sparsity of SLAM has already produced a significant savings here compared to the $\bigO( (d_x n_x + d_l n_l)^3 )$ bound which would be achieved for a fully-dense system.
Nonetheless, computation is still cubic with the number of poses $n_x$, and thus reducing the number of pose nodes promises to reduce computation significantly.

\subsection{Keyframing}

Keyframing approaches aim to do exactly that.
By selecting only a subset of measurement frames to incorporate in the SLAM system, the number of included poses is reduced, ultimately leading to a cubic reduction in EC.
Here, a \emph{frame} can refer to an image frame in a video stream, or more generally any set of measurements produced at the same time and corresponding to a single robot pose.

Letting $G_k(r)$ refer to the graph produced from $G$ by keeping only $n_x / r$ pose nodes, it is easy to see that
\begin{equation}\label{eq:complexity_Gk}
  \mc{C}(G_k(r), \bar{\mc{P}}) \sim \Big(d_l n_l + d_x \frac{n_x}{r} \Big) \Big(d_x \frac{n_x}{r} \Big)^2
\end{equation}

Compared to \eqref{eq:complexity_G}, \eqref{eq:complexity_Gk} shows an asymptotic complexity reduction of between $r^2$ and $r^3$.
It will be demonstrated numerically that this reduction in practice can be much closer to $r^3$.
As noted by Ila et al.\ \cite{ila2010information}, reducing the number of redundant poses in the SLAM system has the added benefit of improving estimator consistency, because in GN approaches each pose corresponds to a first-order noise propagation.
All in all, this makes keyframing an ideal computation-reduction strategy if intermediate pose estimates are not of direct interest.

\subsection{Decimation}\label{s:decimation}

In contrast to keyframing, decimation can be applied \emph{per measurement}, and is particularly useful if the set of poses in the graph is fixed -- i.e., if estimation of the pose at each frame is of direct interest.
Algorithm \ref{alg:ndec} defines the decimation rule used here.
By taking every $r$-th observation of each landmark in a \emph{non-aligned} fashion, decimated graphs demonstrate a pattern of offset partitioning as shown in Figure \ref{fig:ndec}.
Here $r$ is the \emph{decimation rate} and $k_j$ refers to the offset associated with a particular landmark $j$.
Usually, the decimation offset $k_j$ is determined by the pose index corresponding to the first available observation of landmark $j$.
For a landmark observed by a set of consecutive frames, this ensures that the maximum number of observations are accepted under the decimation constraint, and allows for landmark measurements to be included in the optimization as early as possible.

\begin{algorithm}[t]
  \caption{Decimation}
  \label{alg:ndec}
  \begin{algorithmic}[1]
    \Procedure{TestDecimate(\textit{Obs}, $r$)}{}
    \State $i \gets \textit{Obs.PoseIdx}$
    \State $k \gets \textit{Offset[Obs.LandmarkIdx]}$
    \If {$i \mod r = k$}
      \State \textit{Keep Obs}
    \Else
      \State \textit{Discard Obs}
    \EndIf
    \EndProcedure
  \end{algorithmic}
\end{algorithm}

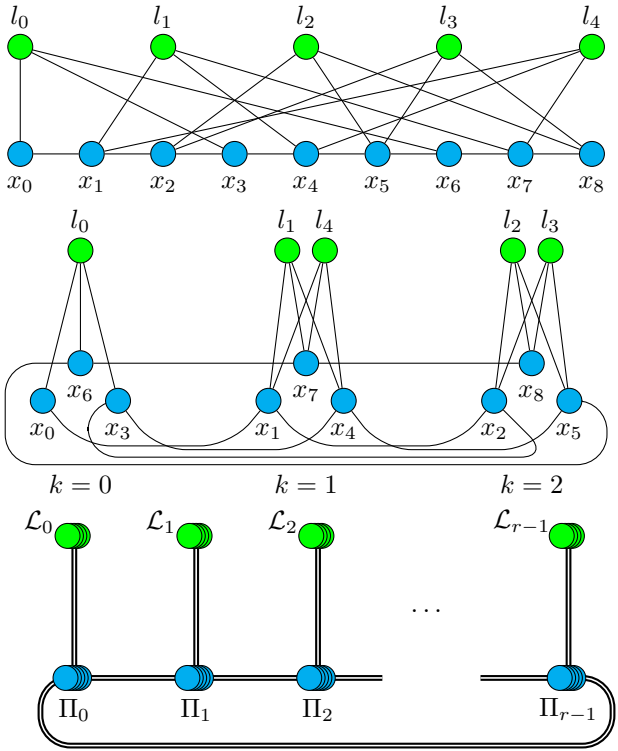
\begin{figure}[t]
  \begin{subfigure}[b]{\columnwidth}
\centerline{    
\begin{tikzpicture}[scale=.95]
  \foreach \i in {0,1,...,8} {
    \node [pnode, anchor=east] (x-\i) at (\i,0) {};
    \node [slabel] at (x-\i) {$x_{\i}$};
  }
  \foreach \j in {0,1,...,4} {
    \node [lnode, anchor=east] (l-\j) at (\j*2,1.5) {};
    \node [nlabel] at (l-\j) {$l_{\j}$};
  }
  \foreach \i[count=\inext from 1] in {0,1,...,7} {
    \draw (x-\i) -- (x-\inext) {};
  }
  \foreach \i in {0,3,...,6} { \draw (x-\i) -- (l-0); }
  \foreach \i in {1,4,...,7} { \draw (x-\i) -- (l-1); }
  \foreach \i in {2,5,...,8} { \draw (x-\i) -- (l-2); }
  \foreach \i in {2,5,...,8} { \draw (x-\i) -- (l-3); }
  \foreach \i in {1,4,...,7} { \draw (x-\i) -- (l-4); }
\end{tikzpicture}}
   \end{subfigure}
  \begin{subfigure}[b]{\columnwidth}
\centerline{    \begin{tikzpicture}
  \node [pnode] (x-0) at (0.00,0.0) {};
  \node [slabel] at (x-0) {$x_0$};
  \node [pnode] (x-3) at (1.0,0.0) {};
  \node [slabel] at (x-3) {$x_3$};
  \node [pnode] (x-6) at (0.5,0.5) {};
  \node [slabel] at (x-6) {$x_6$};
  \node [lnode] (l-0) at (0.5,2.0) {};
  \node [nlabel] at (l-0) {$l_0$};
  \node [] at (0.5, -1.1) {$k = 0$};

  \node [pnode] (x-1) at (3.0,0.0) {};
  \node [slabel] at (x-1) {$x_1$};
  \node [pnode] (x-4) at (4.0,0.0) {};
  \node [slabel] at (x-4) {$x_4$};
  \node [pnode] (x-7) at (3.5,0.5) {};
  \node [slabel] at (x-7) {$x_7$};
  \node [lnode] (l-1) at (3.25,2.0) {};
  \node [nlabel] at (l-1) {$l_1$};
  \node [lnode] (l-4) at (3.75,2.0) {};
  \node [nlabel] at (l-4) {$l_4$};
  \node [] at (3.5, -1.1) {$k = 1$};

  \node [pnode] (x-2) at (6.0,0.0) {};
  \node [slabel] at (x-2) {$x_2$};
  \node [pnode] (x-5) at (7.0,0.0) {};
  \node [slabel] at (x-5) {$x_5$};
  \node [pnode] (x-8) at (6.5,0.5) {};
  \node [slabel] at (x-8) {$x_8$};
  \node [lnode] (l-2) at (6.25,2.0) {};
  \node [nlabel] at (l-2) {$l_2$};
  \node [lnode] (l-3) at (6.75,2.0) {};
  \node [nlabel] at (l-3) {$l_3$};
  \node [] at (6.5, -1.1) {$k = 2$};

  \begin{pgfonlayer}{background}
    \foreach \i in {0,3,...,6} { \draw (x-\i) -- (l-0); }
    \foreach \i in {1,4,...,7} { \draw (x-\i) -- (l-1); }
    \foreach \i in {2,5,...,8} { \draw (x-\i) -- (l-2); }
    \foreach \i in {2,5,...,8} { \draw (x-\i) -- (l-3); }
    \foreach \i in {1,4,...,7} { \draw (x-\i) -- (l-4); }

    \draw (x-6) -- (x-7);
    \draw (x-7) -- (x-8);
    \path [draw, rounded corners=4mm] (x-0) -- ++(0.5, -0.6) -- ++(2.0, 0.0) -- (x-1);
    \path [draw, rounded corners=4mm] (x-1) -- ++(0.5, -0.6) -- ++(2.0, 0.0) -- (x-2);
    \path [draw, rounded corners=4mm] (x-3) -- ++(0.5, -0.65) -- ++(2.0, 0.0) -- (x-4);
    \path [draw, rounded corners=4mm] (x-4) -- ++(0.5, -0.65) -- ++(2.0, 0.0) -- (x-5);
    \path [draw, rounded corners=4mm] (x-2) -- (6.75, -0.75) -- (0.6, -0.75) -- (0.6, 0.0) -- (x-3);
    \path [draw, rounded corners=4mm] (x-5) -- (7.5, 0) -- (7.5, -0.85) -- (-0.5, -0.85) -- (-0.5, 0.5) -- (x-6);
  \end{pgfonlayer}
\end{tikzpicture}}
  \end{subfigure}
  \begin{subfigure}[b]{\columnwidth}
\centerline{    \tikzstyle{subgraph} = [draw,minimum width=1.45cm,minimum height=3.25cm]
\def\s{1.8}  

\begin{tikzpicture}[scale=.9]
  \foreach \k in {0,1,...,1} {
    \draw [double,thick] (\k*\s,-1) -- (\k*\s+\s,-1) {};
  }
  \draw [double,thick] (2*\s,-1) -- ++(0.95,0) {};
  \draw [double,thick] (4*\s,-1) -- ++(-1.2,0) {};
  \path [draw,double,thick,rounded corners=4mm] (0,-1) -- ++(-0.5,0) -- ++(0,-1.0) -- ++(4*\s+1.25,0) -- ++(0,1) -- (4*\s,-1) {};
  \foreach \k in {0,1,...,2} {
    \draw [double,thick] (\k*\s,-1.0) -- (\k*\s,1.1) {};
    \node at (\k*\s-0.52,1.3) {$\mc{L}_{\k}$};
    \node at (\k*\s,-1.47) {$\Pi_{\k}$};
    \foreach \i in {0,1,...,4} {
      \node [pnode] at (\k*\s+0.07- \i*0.05,-1.0) {};
    }
    \foreach \j in {0,1,...,3} {
      \node [lnode] at (\k*\s+0.05 - \j*0.05,1.1) {};
    }
  }
  \node at (3*\s-0.2,0) {$\mathbf{\ldots}$};
  \draw [double,thick] (4*\s+0.1,-1.0) -- (4*\s+0.1,1.1) {};
  \node at (4*\s-0.59,1.3) {$\mc{L}_{r-1}$};
  \node at (4*\s+0.1,-1.47) {$\Pi_{r-1}$};
  \foreach \i in {0,1,...,4} {
    \node [pnode] at (4*\s+0.17- \i*0.05,-1.0) {};
  }
  \foreach \j in {0,1,...,3} {
    \node [lnode] at (4*\s+0.15 - \j*0.05,1.1) {};
  }
\end{tikzpicture}}
  \end{subfigure}
  \caption{
    [top] Decimated SLAM example with $r=3$.
      Each landmark $l_j$ is associated (arbitrarily) with a particular offset $k_j \in \{0,1,\ldots,r-1\}$, which together with the rate $r$ defines which observations are accepted.
      Here, $l_0$ has offset 0, $l_1$ and $l_4$ have offset 1, and $l_2$ and $l_3$ have offset 2.
    [middle] The same decimated graph, redrawn to highlight how landmarks are grouped by decimation offset, partitioning the overall graph.
    [bottom] A generalized representation of decimated structure.
    Note that the $r$~subgraphs have limited inter-connectivity arising only from odometry edges.
  }
  \label{fig:ndec}
  \vspace{-0.25cm}
\end{figure}

Applying the elimination ordering $\bar{\mc{P}}$ to this graph $G_{\ttt{dec}}$ yields
\begin{equation}\label{eq:complexity_Gdec}
  \mc{C}(G_{\ttt{dec}}, \bar{\mc{P}}) \sim (n_l d_l + 9 n_x d_x) \Big(d_x \frac{n_x}{r} \Big)^2
\end{equation}
To obtain \eqref{eq:complexity_Gdec}, we follow ordering $\bar{\mc{P}}$ and eliminate the $n_l$ landmarks first.
This involves a total of $n_l d_l (d_l + n_x d_x / r)^2 \sim n_l d_l (n_x d_x)^2 / r^2$ operations.
Assuming the worst-case, this leaves each $\Pi_k$ as a clique, with additional (assumed complete) connections to its neighboring cliques $\Pi_{k-1}$ and $\Pi_{k+1}$.
Eliminating each of these cliques one-by-one involves computation upper-bounded by
\begin{align}
  \sum_{k=0}^{r-1} \sum_{i=1}^{\lvert \Pi_k \rvert} d_x \Big( d_x + 3 d_x \frac{n_x}{r} \Big)^2  &\sim r \frac{n_x}{r} d_x \Big( 3 d_x \frac{n_x}{r} \Big)^2 \\
  &= \frac{9}{r^2} (d_x n_x)^3
\end{align}
and adding this computation to the previous step gives \eqref{eq:complexity_Gdec}.

Thus, we see that decimation produces a partitioned super-structure which reduces graph complexity asymptotically by $r^2 / 9$ compared to \eqref{eq:complexity_G}.
Comparing to keyframing \eqref{eq:complexity_Gk}, decimation is clearly less effective at reducing computation.
Nonetheless, decimation can have other advantages, particularly if estimating the full trajectory (including intermediate poses) is of direct interest.
A full comparison of keyframing and decimation is outside the scope of this paper.

The experimental results in the next section emphasize that it is the unique partitioned structure shown in Figure~\ref{fig:ndec}, rather than simply the reduction in number of measurements, which produces this computational savings.

\subsection{Experimental Results}

A suite of simulation experiments were performed to verify the analytic results discussed above in a full 3D, incremental visual SLAM setting.
In the simulation, a robot drives a sinusoidal trajectory, observing nearby landmarks according via a monocular visual sensor.
At each step of the simulation, a new pose node is added to the graph, and newly-triangulated landmarks are added to the graph.
Because of the under-rank nature of monocular measurements, landmarks are not initialized (i.e.\ added to the graph) until they have been observed a minimum number of times.
Poses are represented as elements of $\SE(3)$, and landmarks as points in $\R^3$.

\begin{figure}[t]
  \centering
  \begin{subfigure}[b]{0.2\textwidth}
    \includegraphics[width=\textwidth]{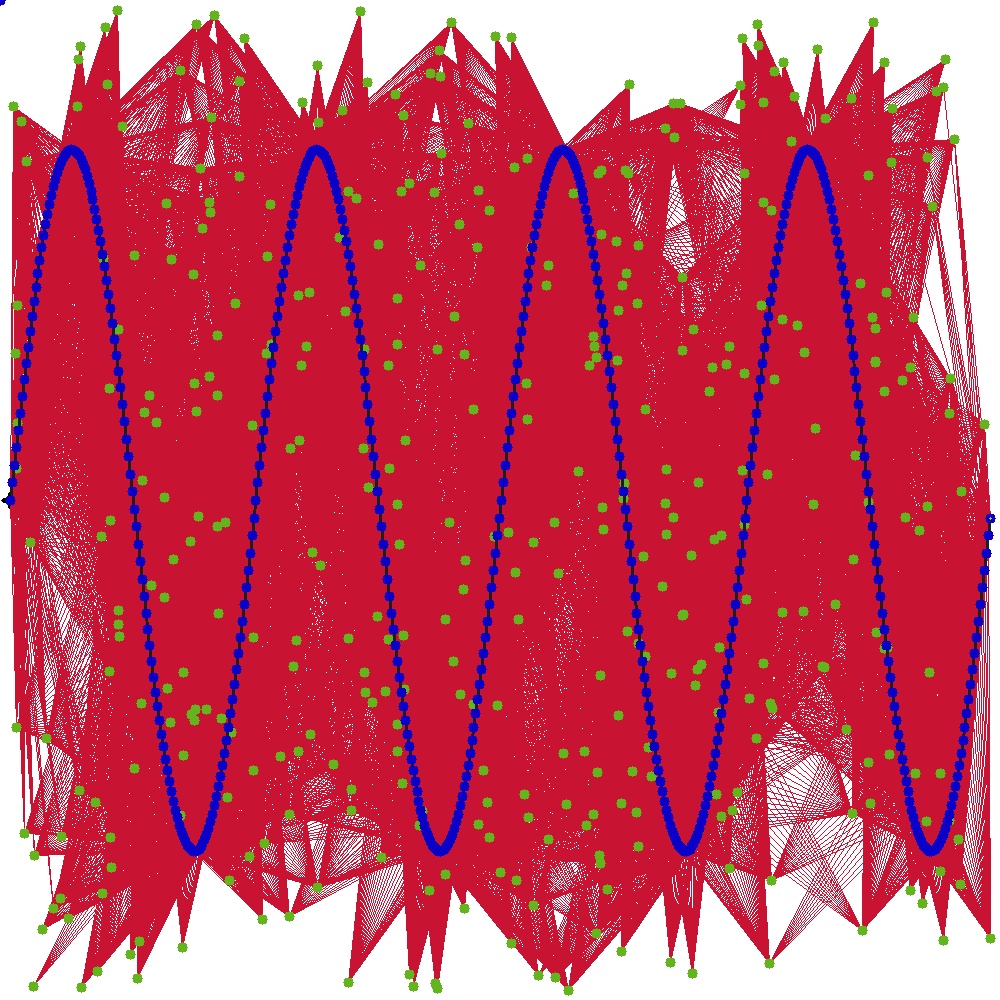}
  \end{subfigure}
  \hfill
  \begin{subfigure}[b]{0.2\textwidth}
    \includegraphics[width=\textwidth]{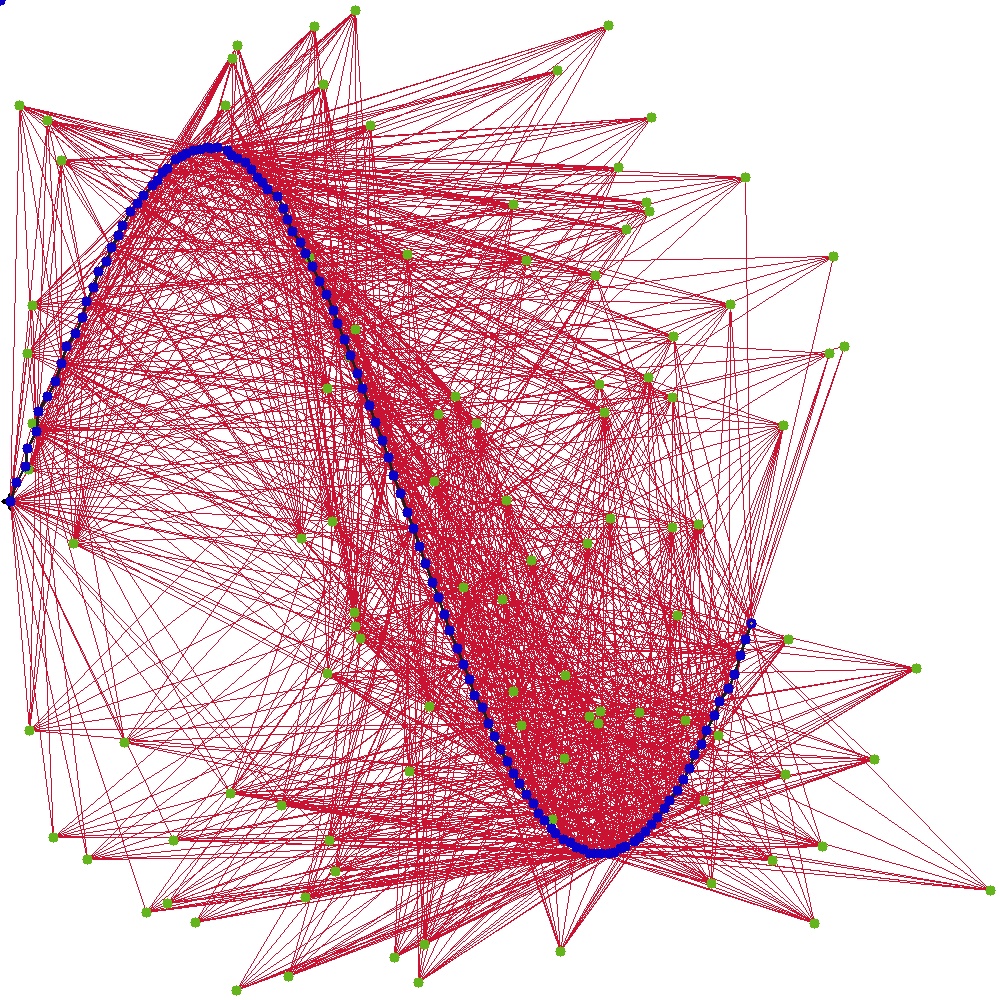}
  \end{subfigure}
  \caption{
    Simulated visual-odometry dataset.
    [left] Top-down view of the full trajectory (blue pose nodes) and landmarks (green).
      The robot moves in a sinusoidal pattern along the horizontal plane, starting from the left and moving right.
      Red edges indicate the full set of monocular vision observations, and blue edges represent pose-pose odometry.
    [right] A partially-completed trajectory, with landmark observations decimated at a rate of $r=4$.
  }
  \label{fig:sim_overview}
  \vspace{-0.25cm}
\end{figure}

\begin{figure*}[t]
	\centering
	\begin{subfigure}[b]{.95\columnwidth}
		\includegraphics[width=\textwidth,trim={3.9cm, 8.5cm, 4.5cm, 8cm},clip]{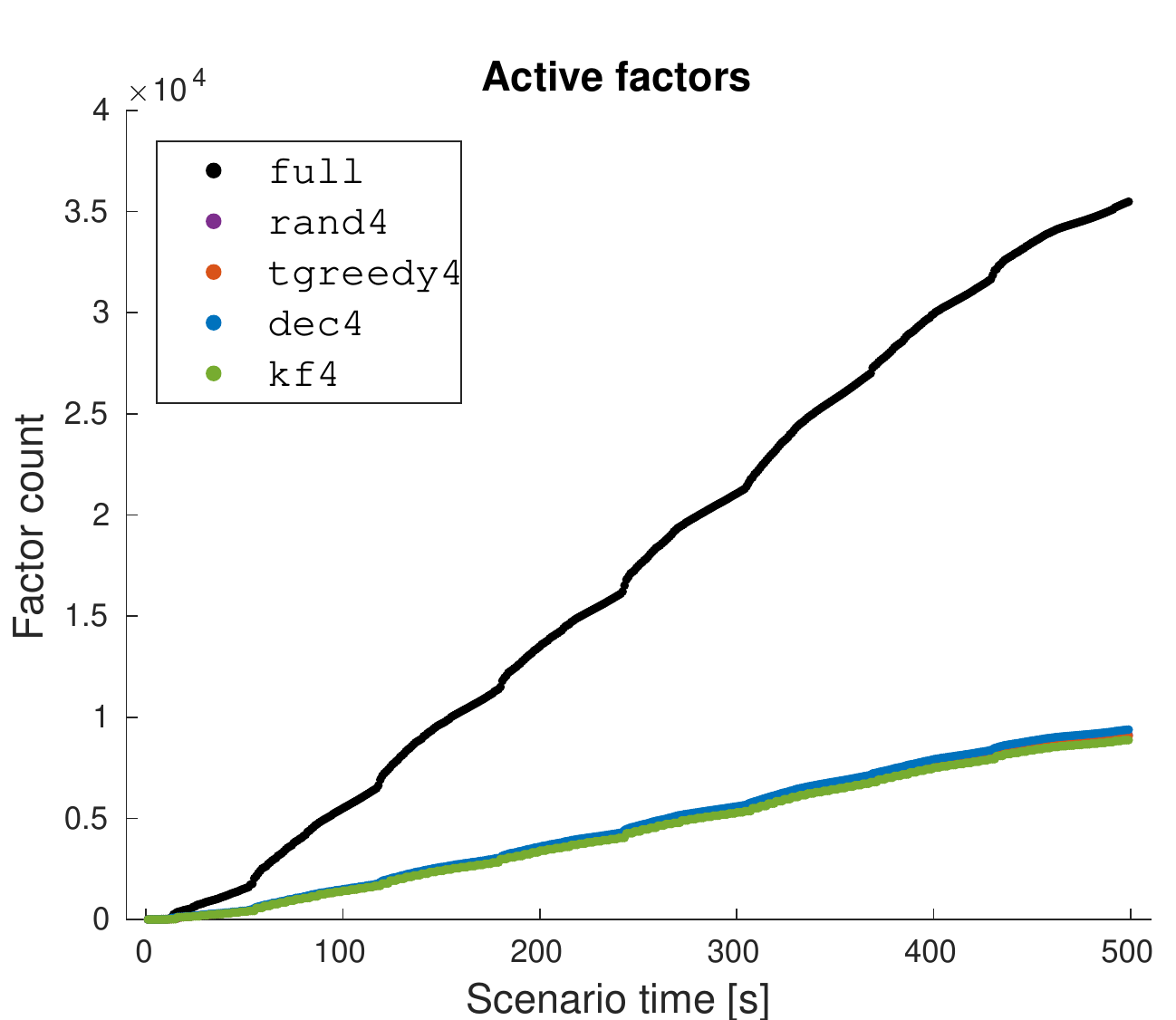}
	\end{subfigure}
\hspace*{.95\columnsep}
	\begin{subfigure}[b]{.95\columnwidth}
		\includegraphics[width=\textwidth,trim={3.9cm, 8.5cm, 4.5cm, 8cm},clip]{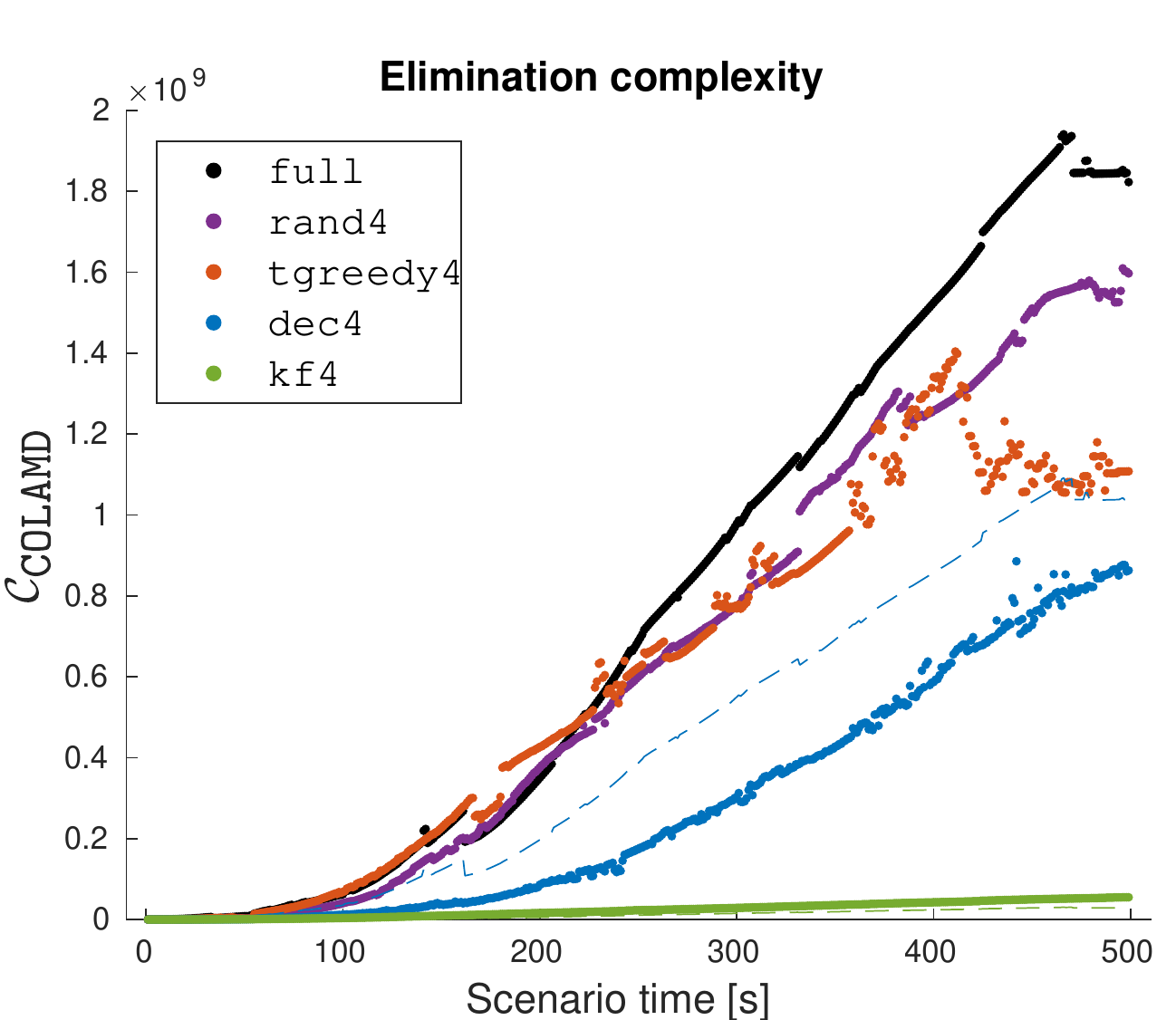}
	\end{subfigure}
	\caption{
		Elimination complexity during a simulated SLAM experiment with $r=4$.
		[left] All three pruning methods maintain a similar number of factors in the graph.
		[right] In all cases, elimination complexity grows over time.
		Though the three pruned estimators maintain a similar number of measurement factors, they have significantly different elimination complexities.
		Because \ttt{rand} prunes without any regard for global structure, it produces underwhelming EC reduction compared to \ttt{dec} or \ttt{kf}.
		The dashed lines demonstrate the predicted complexities for \ttt{dec} and \ttt{kf}, based on scalings of \ttt{full} derived from \eqref{eq:complexity_Gdec} and \eqref{eq:complexity_Gk}.
		The actual complexities match these predictions well, validating the method of analysis presented in this paper.
	}
	\label{fig:results_r4}
	\vspace{-0.25cm}
\end{figure*}

\begin{figure*}[t]
	\centering
	\begin{subfigure}[b]{.95\columnwidth}
		\includegraphics[width=\textwidth,trim={3.9cm, 8.5cm, 4.5cm, 8cm},clip]{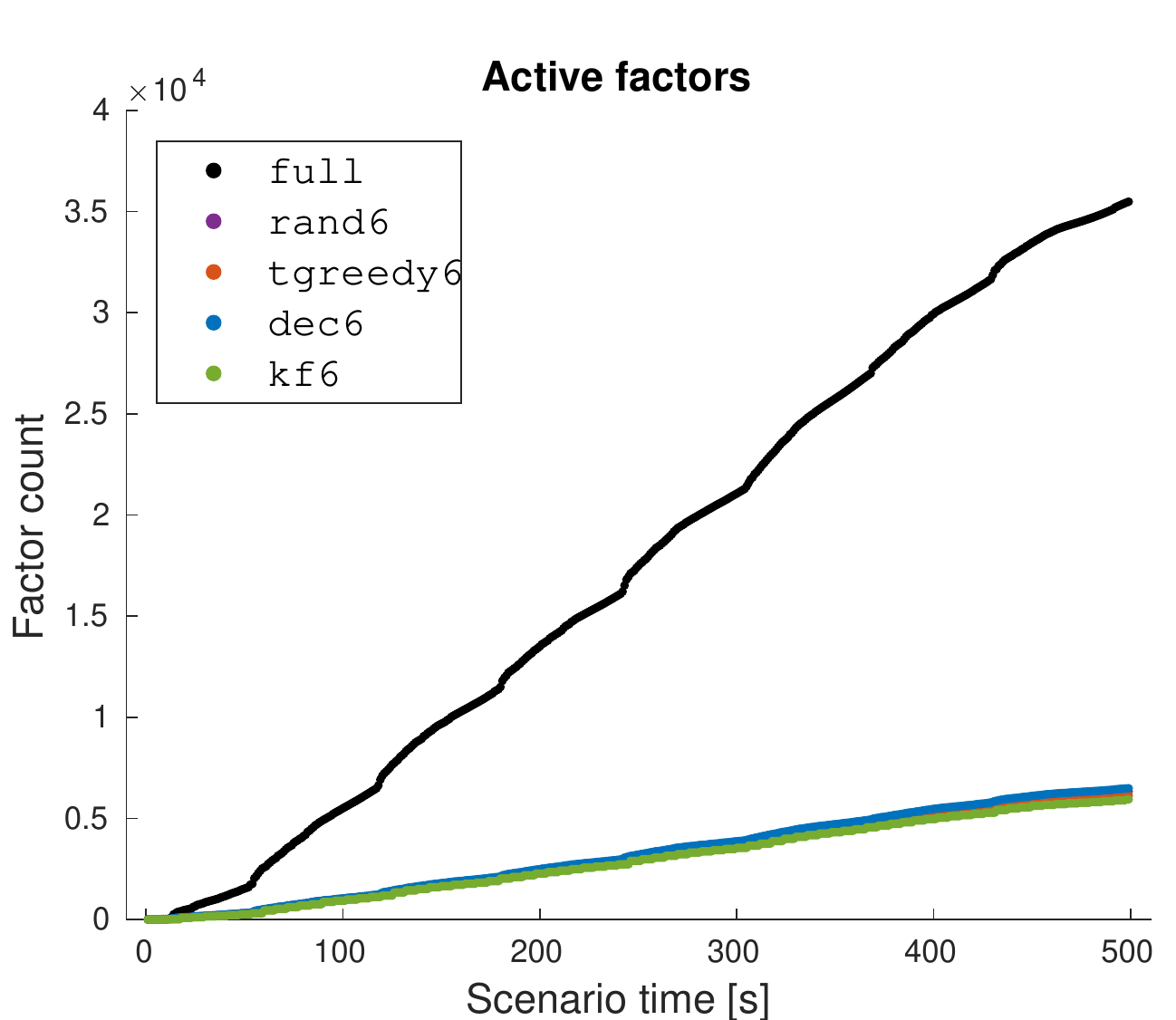}
	\end{subfigure}
\hspace*{.95\columnsep}
	\begin{subfigure}[b]{.95\columnwidth}
		\includegraphics[width=\textwidth,trim={3.9cm, 8.5cm, 4.5cm, 8cm},clip]{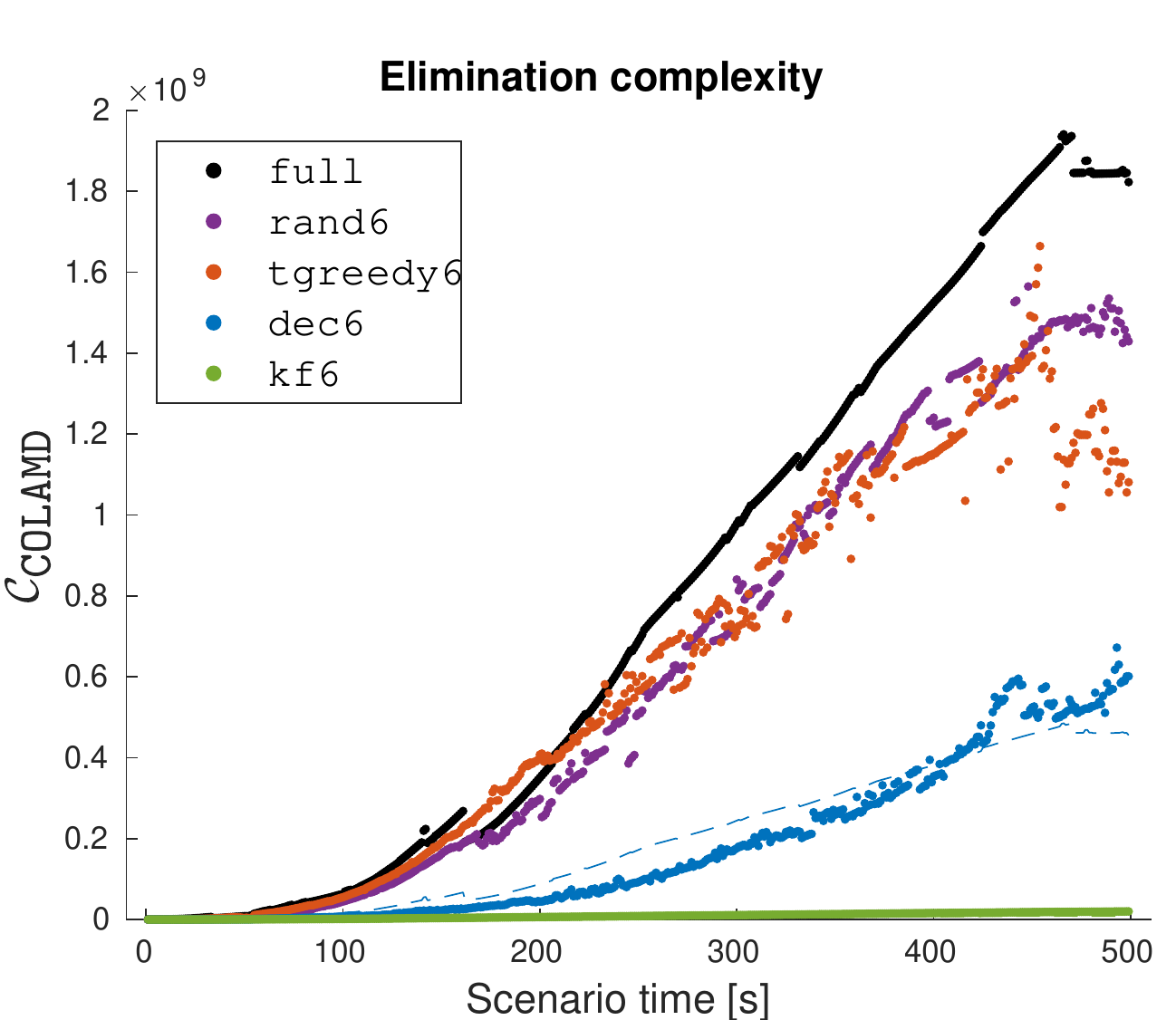}
	\end{subfigure}
	\caption{
		Elimination complexity during a simulated SLAM experiment with $r=6$.
		The trends here are very similar to those demonstrated in Figure \ref{fig:results_r4}.
	}
	\label{fig:results_r6}
	\vspace{-0.25cm}
\end{figure*}

The full trajectory, landmark distribution, and set of available landmark observations are shown in Figure \ref{fig:sim_overview}.
As this simulation accounts for realistic sensor limitations, any given landmark is only observed by a subset of robot poses.
This means that the complexity results derived in Section \ref{s:primitives}, which assume all landmarks are observable from all poses, may provide significant \emph{overestimates} of the realized complexity $\mc{C}_{\ttt{COLAMD}}$.

Several pruning strategies are evaluated here, all parameterized by pruning rate $r$.
\begin{itemize}
  \item \ttt{rand}: Random pruning, parameterized to remove a similar number of observations to the other  methods.
  \item \ttt{tgreedy}: Greedy algorithm of \cite{khosoussi2016maximizing} which attempts to maximize the number of spanning trees in the graph.
  \item \ttt{kf}: Simple keyframing strategy which represents only every $r$-th timestep in the optimization.
  \item \ttt{dec}: Decimation strategy described in Section \ref{s:decimation}, which always accepts the first available observation of each landmark.
\end{itemize}
As a comparison, results with no pruning are shown as \ttt{full}.
In all cases, only visual observations (pose-landmark edges) are considered for pruning.

The goal of these experiments is to the verify the complexity reduction estimates derived in the previous section.
As expected, the predicted complexity reduction of both keyframing and decimation methods are well-supported by the experimental results shown in Figures \ref{fig:results_r4} and \ref{fig:results_r6}.
In each case, the plotted dashed lines represent a complexity ``prediction'' produced by a simple scaling of the elimination complexity of $\ttt{full}$.
For keyframing this prediction is scaled by $1 / r^3$, and for decimation by $9 / r^2$, based on \eqref{eq:complexity_Gk} and \eqref{eq:complexity_Gdec} respectively.
As can be seen, the realized complexity follows this prediction well, verifying the method of analysis presented in Section \ref{s:primitives}.

The average iSAM2 update computation times are in Table \ref{table:isam_perf}, which combined with the EC plots in Figures \ref{fig:results_r4} and \ref{fig:results_r6}, shows that the  savings achieved by \ttt{rand} and \ttt{tgreedy} are relatively small and do not scale with increased pruning.
In contrast, \ttt{dec} and \ttt{kf} provide significant reduction at $r = 4$ and continue to improve at $r=6$.

\begin{table}[t]
  \centering
  \caption{
    iSAM2 performance in simulation shown in Figure \ref{fig:sim_overview}.
    Note that \texttt{rand} and \texttt{tgreedy} fail to improve mean iSAM2 update times significantly between $r=4$ and $r=6$, while \texttt{dec} and \texttt{kf} continue to reduce computation with increased pruning.
  }
  \label{table:isam_perf}
  \begin{tabular}{c c}
    \toprule
    \textbf{Method} & \textbf{Avg. iSAM2 update time [s]} \\
    \midrule
    \ttt{full} & 0.205 \\
    \\
    \ttt{rand4} & 0.154 \\
    \ttt{tgreedy4} & 0.138 \\
    \ttt{dec4} & 0.087 \\
    \ttt{kf4} & \textbf{0.023} \\
    \\
    \ttt{rand6} &  0.131 \\
    \ttt{tgreedy6} & 0.137 \\
    \ttt{dec6} & 0.062 \\
    \ttt{kf6} & \textbf{0.012} \\
    \midrule
  \end{tabular}
  \vspace{-0.5cm}
\end{table}

\section{Conclusions}

Many existing measurement selection techniques focus narrowly on edge count reduction \cite{khosoussi2016maximizing,huang2013consistent,dissanayake2000computationally} or locally-sparse structure \cite{carlevaris2013long,mazuran2016nonlinear}.
By neglecting global structure, these methods can produce limited computation reduction even after aggressive sparsification.
This paper proposes the use of \emph{elimination complexity} (EC) as a link between graph structure and computation, and demonstrates how simple heuristics like decimation and keyframing produce dramatic computation savings.

As an analytic tool, the EC framework is used to predict asymptotic computation reduction scaling with $r^2 / 9$ and $r^3$ for decimation and keyframing, respectively.
These predictions are confirmed numerically, and shown to far outperform structurally na\"{\i}ve methods which remove the same number of edges.
In addition to the fact that many sophisticated selection approaches can be computationally impractical for high-rate, realtime use, this demonstrates that they also may not be as effective at reducing computation as these much simpler policies.

Ultimately, this motivates the search for new measurement selection strategies which directly and efficiently reduce EC, rather than focusing naively on edge count or local sparsification.
As was seen here, EC itself can be used to evaluate existing strategies, and provides the link between graph structure and computation that to the best of the authors' knowledge has been lacking from the SLAM literature.

Because it can be evaluated from and updated alongside the iSAM2 Bayes Tree, EC can also be applicable for adaptation of computation-management policies online.
For example, if a robot using a sliding-window or fixed-lag estimator \cite{chiu2013robust,steiner2017vision,sibley2010sliding} enters a new environment with relatively fewer landmarks, the active SLAM graph will decrease in complexity.
By monitoring EC, the estimator could choose to increase the window length to fill its computation budget and ensure sufficient landmarks are tracked at all times.

\section*{Acknowledgment}

The authors thank Dr.\ Kasra Khosoussi for the productive conversations and input over the course of this work.
Furthermore, we would like to thank DARPA for supporting this research and additionally Julius Rose from Draper and Professor Nicholas Roy from MIT for their leadership and general support.

\balance

\bibliographystyle{IEEEtran}
\bibliography{IEEEabrv,ref}


\end{document}